\documentclass[letterpaper, 10 pt, conference]{ieeeconf}
\IEEEoverridecommandlockouts                              
\usepackage{amsmath,bm,amsfonts,cases}
\usepackage{shortcuts,subfiles,balance}
\usepackage{booktabs,xcolor}
\usepackage{graphicx}
\usepackage{float}
\usepackage{url}
\usepackage{optidef}    
\usepackage{mathrsfs}
\usepackage[hidelinks]{hyperref} 
\usepackage{multirow}
\usepackage{colortbl}
\usepackage{siunitx}    

\usepackage{algorithm}
\usepackage{algpseudocode}

\usepackage[symbol]{footmisc}

{
    \newtheorem{proposition}{Proposition}
}

\let\labelindent\relax
\usepackage{enumitem}

\usepackage{tabularx}
\usepackage[format=plain,font=footnotesize,labelsep=colon]{caption}

\title{\LARGE \bf Generate, Track, Improve: Perceptive Multi-Skill Humanoid Locomotion with RL-Fine-Tuned Motion Generators}
\author{Zachary Olkin, William D. Compton, Aaron D. Ames$^{}$
\thanks{The authors are with the Department of Control and Dynamical Systems and the Department of Mechanical and Civil Engineering at the California Institute of Technology. 
This research is supported by the Technology Innovation Institute (TII).}
}

\begin{document}
\bstctlcite{IEEEexample:BSTcontrol}
\twocolumn[{%
  \renewcommand\twocolumn[1][]{#1}%
  \maketitle
  \begin{center}
    \includegraphics[width=1.0\linewidth]{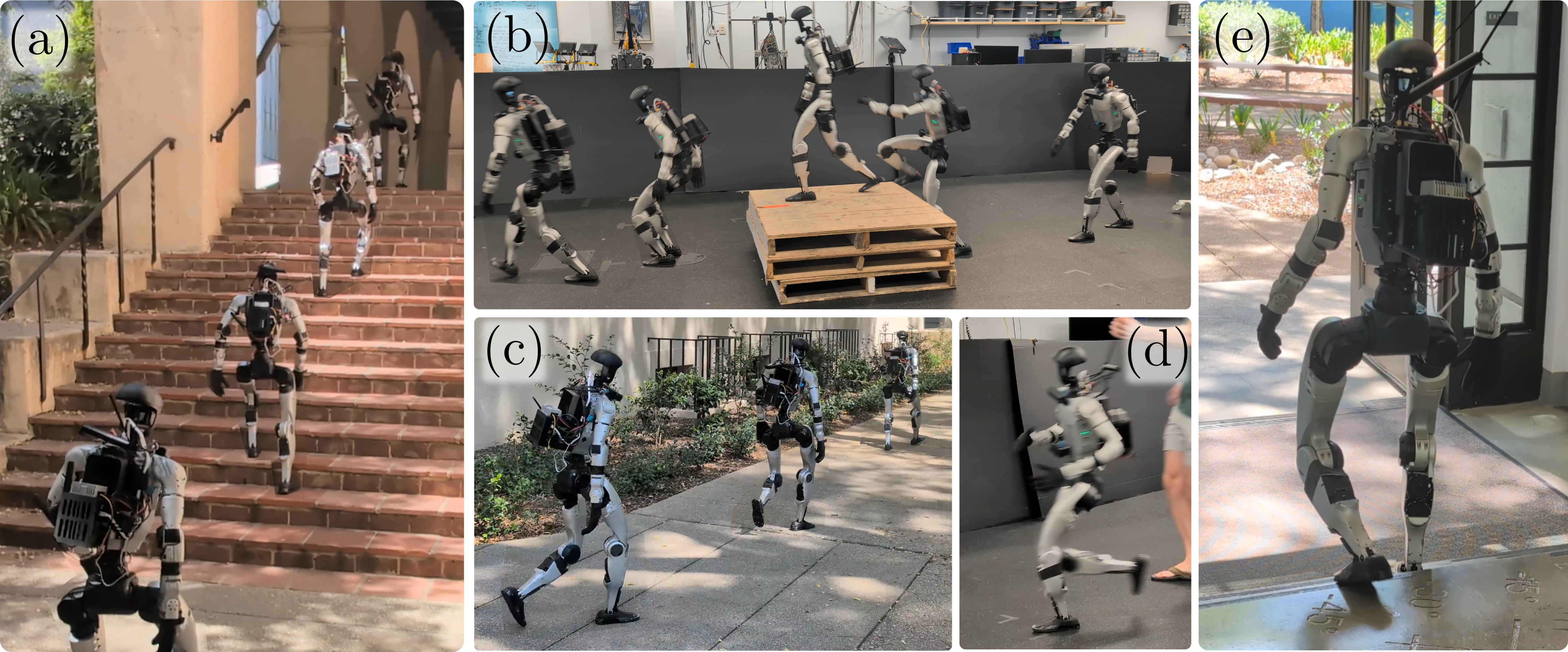}
    \captionof{figure}{Demonstration of our perceptive control policies working across multiple terrains in indoor and outdoor settings. (a) Ascending a 15 step real world staircase including transitions into and out of the stair climb. (b) Jumping onto a box, walking across it, and jumping off the box. (c) Outdoor walking locomotion including turning onto a path. (d) Running across an indoor space. (e) Descending a real world stair case. }
    \label{fig:hero}
  \end{center}
}]
{\renewcommand{\thefootnote}{}%
 \footnotetext{The authors are with the Department of Control and Dynamical Systems at the California Institute of Technology. This research is supported by the Technology Innovation Institute (TII).}}
 
\begin{abstract}
General purpose humanoids require locomotion controllers that are multi-skill, perceptive, dynamic, and robust enough to go anywhere humans can. In this work, we present a two layer locomotion architecture: (1) a perceptive flow matching motion generator plans whole body trajectories from raw depth images while a (2) perceptive tracking policy trained with control-guided RL follows these motions. Both policies are trained on a library of terrain consistent motion clips created with dynamically optimized human data which yields both accurate velocity tracking and terrain consistent references. Our central contribution is a simple yet effective off-policy RL fine tuning loop that improves the motion generator. A structured search method is used with the generator to gather data for advantage weighted regression. This off-policy loop is much more sample efficient than on-policy residual fine tuning and improves terrain consistency on unseen geometries and skill compositions. We find that successful terrain traversals increased by up to 25 percentage points and skill selection improved by up to 80 percentage points. By using raw depth images to perceive the environment no odometry or height maps are needed, and outdoor deployment is easy. With two cameras, the policy can see terrain coming from further away and adjust its velocity regardless of the commanded speed so it can traverse the terrain. A single policy pair enables a Unitree G1 humanoid to walk, run, stand, jump on and off of boxes, and traverse stairs in outdoor environments. Project page: \textcolor{blue}{\url{https://zolkin1.github.io/generate-track-improve/}}.
\end{abstract}

\section{Introduction}
Between the diversity of skills required and the inherently perceptive nature of real world locomotion, humanoid robots have yet to be seen reliably operating in the real world. 
In recent years many humanoid locomotion modalities, such as walking \cite{radosavovic_real-world_2024, li_clf-rl_2026, siekmann_sim--real_2021}, running \cite{olkin_chasing_2026, li_reinforcement_2024}, parkour \cite{wu_perceptive_2026, xu_parc_2025, zhuang_humanoid_2024}, stair traversals \cite{zhang_rpl_2026, compton_terrain_2026, long_learning_2024}, and much more have been studied. Yet, even with this explosion in capability, there are many challenges in humanoid whole body control and locomotion still yet to be solved. For example, many terrain aware locomotion policies \cite{zhang_rpl_2026, compton_terrain_2026, wang_beamdojo_2025, dai_walk_2026, crismariu_march_2026} are not built to be extensible to general whole body control or multi-speed (walking, running, etc...) locomotion. This makes it difficult for those robots to go anywhere in the real world at a commanded speed or extended to interact with the environment like opening doors or pressing elevator buttons. On the other hand, multi-purpose whole body controllers \cite{luo_sonic_2026, yin_unitracker_2025} have emerged that can track many motions, but these are lacking the perception and terrain aware motions needed to traverse the real world. In this work we bridge these ideas by making a pair of perceptive policies that together allow the robot to generate terrain aware references and track those references in real time. 

To tackle this challenge, we leverage a layered architecture with a motion generation layer interfacing into a tracking layer where both policies are perceptive. Then, we also present a simple but highly effective RL fine tuning loop that can improve the mode selection and terrain consistency of the motion generator policy. This generator and tracker paradigm is not the only way to create a multi-skill and perceptive policy; others have used multi-expert distillation \cite{rudin_parkour_2025, wu_perceptive_2026}, or created generators that produce both kinematic trajectories and actions  \cite{huang_diffuse-cloc_2025}. In comparison to these other methods, the generator and tracker paradigm leads to modularity through a layered architecture and the ability to easily add additional skills, provided a general enough tracking controller. This is ideal for scaling up generality in humanoids. Further, this provides a natural computation split, allowing a relatively light weight tracking controller to be used at a high speed while the larger motion generator can be queried at a slower speed.

\subsection{Related Works}
Humanoid whole body and locomotion control uses RL methods. Controllers are generally trained in simulation \cite{rudin_learning_2022} with PPO \cite{schulman_proximal_2017}. Dynamic and human like motions are commonly achieved with ``mimic" style approaches \cite{peng_deepmimic_2018, liao_beyondmimic_2025, sleiman_zest_2026, luo_sonic_2026}. These methods offer great performance, clean results, and minimal reward tuning at the mere cost of obtaining a reference trajectory. Human data \cite{yang_omniretarget_2025, liao_beyondmimic_2025, xie_kungfubot_2025, he_asap_2025, ji_exbody2_2025}, reduced order models \cite{lee_integrating_2024, li_clf-rl_2026, green_learning_2021, batke_optimizing_2022, compton_terrain_2026}, optimized dynamic trajectories \cite{li_clf-rl_2026, olkin_chasing_2026, esteban_shooting_2026, li_reinforcement_2021, liu_opt2skill_2025}, and animation \cite{grandia_design_2024, sleiman_zest_2026} have all been used to synthesize reference trajectories. Many of these mimic style approaches create controllers that can do motion ``play back" but still lack more general purpose skills or an interface for autonomy. In locomotion, achieving this autonomy interface is commonly done through velocity conditioning \cite{li_reinforcement_2024}, periodic orbit creation \cite{olkin_chasing_2026, li_clf-rl_2026}, or reduced order models that can be queried in training for the desired motion at any point in time \cite{lee_integrating_2024, compton_terrain_2026}; each of which allow the robot to track a commanded velocity, which is something a higher level navigation policy might output.

Even with some methods using perception to traverse terrain, perceptive locomotion is still an evolving research direction \cite{gu_evolution_2026}. Perceptive humanoid locomotion has a few major paradigms such as use of a height field \cite{zhang_learning_2026, wang_beamdojo_2025, long_learning_2024, sun_learning_2025}, or using direct sensor measurements \cite{wu_perceptive_2026, zhang_rpl_2026, compton_terrain_2026, rudin_parkour_2025}. Use of a height field creates difficulty in deployment as it requires a height scanning package and odometry in general \cite{miki_elevation_2022}. Using the raw depth camera inputs removes this blocker at the expense of needing to simulate the sensor and requiring that the network outputs are observable given the view of the cameras.

The generator and tracker paradigm is gaining momentum in the humanoid control space, but is still a relatively nascent architecture. SONIC \cite{luo_sonic_2026} achieved this paradigm with a kinematic motion generator paired with their universal tracking policy. Critically, their generator was not perceptive and designed for just flat ground traversal. PARC \cite{xu_parc_2025} proposes a perceptive tracker generator pipeline that can generate novel motions from a finite starting set. The pipeline is not shown on hardware, being restricted to simulated characters, thus leaving its applicability to real robots in question. Even so, the resulting motions are impressive, but the method requires kinematic and dynamic correction steps which in conjunction with the diffusion policy training takes one month for three iterations of training. This is much slower than our proposed method. The authors of \cite{zhang_learning_2026} demonstrate a perceptive generator and tracker for humanoids, but their method requires odometry and a height scan, does not show accurate velocity tracking or outdoor experiments and requires fine tuning of the tracker. Our method uses raw depth cameras, no odometry, shows accurate velocity tracking, and critically we provide a pipeline for improving the generator with RL which fixes issues that tracker fine tuning can't fix.

The idea of using RL for fine tuning flow matching and diffusion policies has been used in VLAs \cite{intelligence__06_2025, chen__textttrl_2025} and for smaller behavior cloning (BC) policies doing manipulation/loco-manipulation on legged and wheeled systems \cite{gu_refine-dp_2026, ankile_residual_2025}. In \cite{gu_refine-dp_2026} the authors use an on-policy algorithm, DPPO \cite{ren_diffusion_2024}, and a reduced action space, which does not directly control the legs, for the diffusion policy outputs. Separately, there is a group of fine tuning methods that freeze the generator and either learn residual actions \cite{ankile_residual_2025} or steer the latent noise \cite{wagenmaker_steering_2025, su_rfs_2026} with \cite{su_rfs_2026} doing both. In \cite{ankile_residual_2025} they use off-policy RL to learn residual actions and apply this to manipulation instead of dynamic locomotion. These methods target manipulation policies with low dimensional action spaces. Our setting differs both in its application to dynamic locomotion but also in that we adjust the original generator's weights.

\subsection{Contributions}
\textbf{Full Architecture:} In this paper we present Generate, Track, Improve: a full pipeline for the creation of a tracker and generator policy pair that enables perceptive humanoid locomotion across multiple terrains. The tracker policy is trained with CLF-RL in simulation to enable precise and robust tracking of the dynamic motions \cite{olkin_chasing_2026}. The generator is a flow matching \cite{lipman_flow_2023} transformer conditioned on raw depth images from multiple cameras and velocity commands which generates a whole body trajectory every 0.24 s and passes this to the tracker. Together, these policies enable a humanoid robot to traverse real world environments with skills such as running, walking, jumping on and off of boxes and traversing a variety of stairs.

\textbf{Data Pipeline:} The data used to build the CLF-RL references comes originally from human motion capture. We utilize an optimization step in the data processing pipeline to create hundreds of dynamically feasible, velocity accurate, and periodic motions for the robot to track. The optimization allows us to re-target the single references we use for terrain interaction into many different references across multiple terrain geometries. Beyond these benefits, we pair these optimized motions with the generative model Motion Bricks \cite{wang_motionbricks_2026} to create motion clips that transition between terrain segments and different velocities. This pairing of Motion Bricks and optimization is shown to provide low velocity tracking error and terrain consistent references.

\textbf{Generator RL Fine Tuning:} After using this data to train the tracker and generator we present an off-policy RL training loop for the generator that iteratively improves its results using its own closed loop rollouts, and a frozen tracker. Given the generator produces just under 3,000 values per plan (for a whole body trajectory), applying naive PPO on the outputs results in extremely inefficient search as independent gaussian noise on those actions does not lead to the structured trajectories that would create better references. In contrast, our method searches through the flow matching initial noise samples and by perturbing the conditioning values in the rollouts, allowing the generator to search other modes without adding noise to the actions. Using this diverse dataset from the rollouts we train a critic and compute advantage labels on the dataset. Using this advantage we perform advantage weighted regression \cite{peng_advantage-weighted_2019} to update the generator via supervised learning. This process is then repeated, resulting in the success rate increases of up to 25 percentage points and improving skill selection by up to 80 percentage points.  To the best of our knowledge, our work is the first time RL fine tuning on a perceptive generative model has been used for dynamic and terrain aware humanoid locomotion.

\textbf{Perceptive Inputs:} The tracker and generator are conditioned directly on the depth images rather than using a height map or learned representation; and no odometry is required. This makes deployment very easy and robust. We use two depth cameras which allow us to see objects before a downward facing camera would see them. This enables the robot to slow down upon approaching obstacles so it can traverse the terrain at the correct speed without needing the velocity command to change. We demonstrate that the robot retains accurate velocity tracking with only minimal deviations to traverse the terrain and how removal of the upper depth camera would cripple this capability. In our method the velocity command can be independent of the terrain and the robot will adjust as necessary. 

\textbf{Hardware Deployment:} We demonstrate this policy extensively on real humanoid robot hardware to walk, run, stand, jump and traverse stairs with ease (see Fig. \ref{fig:hero}). We do this both in lab settings and outdoors in real world environments, demonstrating real world capabilities rarely shown in other perceptive dynamic humanoid locomotion policies.

\section{Methods}
\begin{figure*}
    \centering
    \includegraphics[width=1.0\linewidth]{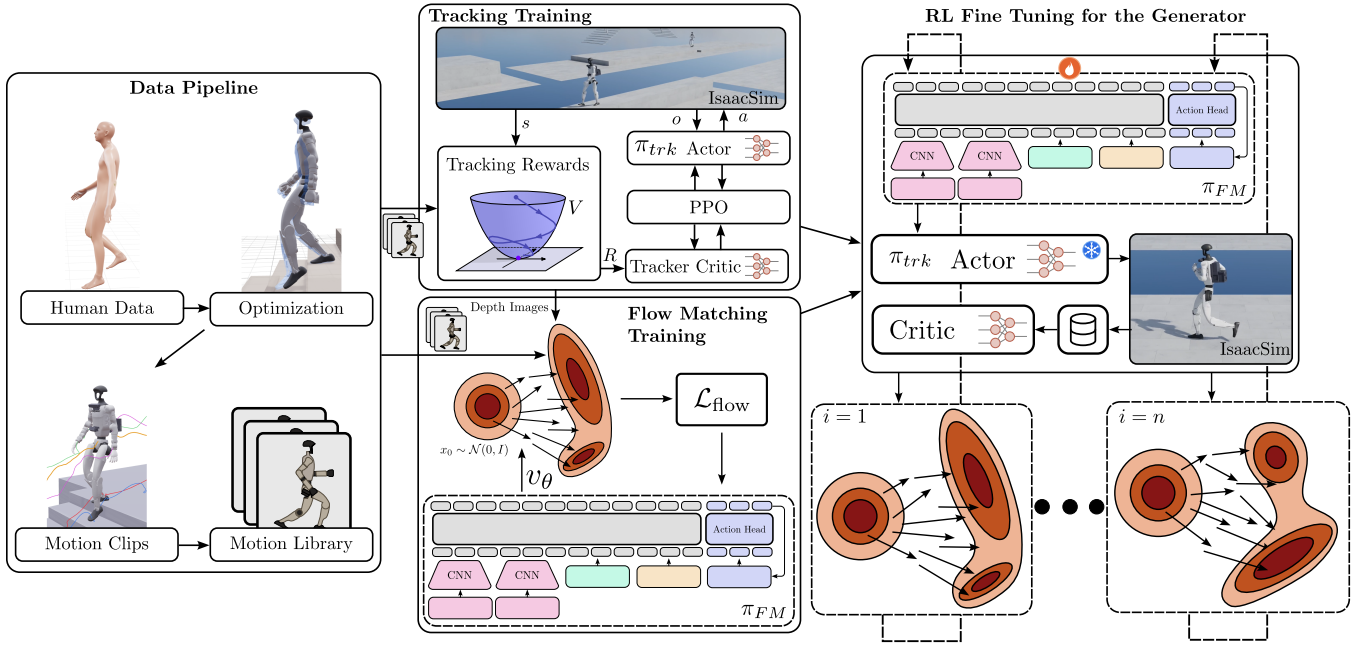}
    \caption{Overview of the presented method. First, human data is translated into motion clips to make a motion library through optimization and generative models. Secondly, a perceptive tracking policy is trained using CLF-RL to track the motion library. Then the motion library and the tracking policy are used to gather depth images that the robot sees at a given state. These depth images are paired with the commanded velocity and the motion library to generate the data for the flow matching generator training. The generator is trained with supervised learning using the flow MSE loss. Then lastly, the generator is fine tuned using off policy RL. The generator and tracker are rolled out in simulation to gather data, which is fed into a buffer. A critic is trained to estimate the infinite horizon reward at each state and then that critic is used to assign advantage labels to the data. The generator is then updated using advantage weighted regression (AWR) to adjust its output distribution. Finally, this policy pair can be deployed on hardware to make the robot perceptively locomote in real world environments.}
    \label{fig:architecture}
    \vspace{-3mm}
\end{figure*}

\subsection{System Overview}
The presented method has four main parts as shown in Fig. \ref{fig:architecture}: (1) the data pipeline, (2) the tracker training, (3) the generator training, and (4) the RL fine tuning of the generator. After these steps two distinct policies are deployed: the fine tuned motion generator and the tracking policy. The tracker runs at 50 Hz as is common in RL control policies while the generator is only queried once every 0.24 seconds (i.e. every 12 control steps). The tracker takes in proprioceptive information and one camera depth image to compute actions while the generator takes in the desired velocity, its previous action, and two depth images to compute a 1.24 second long whole body trajectory. The data pipeline provides the terrain consistent reference motions for the tracker and generator to train on while the fine tuning improves the generator after it is already trained.

\subsection{Motion Clip Creation}
Our method requires pre-made terrain aware motion clips for the tracking policy. The details of their creation are given here. We start with human data from the bones-seed data set \cite{bones_studio_bones-seed_2026}. We use this data to create optimized human references similar to \cite{olkin_chasing_2026, esteban_shooting_2026}. 252 flat ground steady-state reference motions are generated including straight walking (forward and backwards), diagonal walking, turning walking, turning in place, side step walking, straight running and running while turning. This gives the robot a wide range of omni-directional locomotion capabilities. Then, additionally, 46 reference motions for jumping onto and off of boxes, as well as going up and down stairs are generated.

The optimization uses a multiple shooting state constrained formulation with MuJoCo \cite{todorov_mujoco_2012} as the dynamics back-end \cite{esteban_shooting_2026}. This design allows for small changes to the contact schedule by differentiating through the MuJoCo contact dynamics while still gaining the stability of the multiple shooting method. The state constraints afforded by the multiple shooting are critical as they allow us to enforce a periodic constraint (i.e. the final state is equal to the initial state) and an average velocity constraint which can adjust the reference motion to exactly the desired speed. 

The terrain based motions in the open source bones-seed dataset don't have the ground truth terrain information. Therefore we start by placing terrain to match the motion. Then we can use the same optimization as before but with the terrain placed in the MuJoCo scene. Once the initial solve is complete, with these default terrain parameters, we have ground truth contact information from the simulation. The ground truth contact information is gathered by thresholding the contact force on the relevant end effector bodies and then filtering out any short patches of contact/no contact. We found this to be more reliable than detecting contact by pure kinematics, and it is a benefit of using a dynamics optimization. With the ground truth contact the robot is kinematically retargeted onto a variety of terrain dimensions, which is especially useful since the dataset does not have every dimension we may see in the real world. To do the kinematic retargeting we compute foot offsets so that in contact the feet are placed on the terrain. Then we smoothly interpolate the offsets from one contact to the next to get a continuous deformation of the end effector trajectory and then we use inverse kinematics to get the joint angles. After the kinematic retargeting onto the other heights we can run the dynamic optimization again to get a feasible and smooth terrain aware motion.

In total, this optimization routine generates steady state flat ground locomotion at a variety of speeds as well as episodic and steady state terrain motions over a variety of geometries. Using these optimized motions, the next step is to generate the full motion clips that will be used by the tracker and generator trainings. These motion clips range anywhere from 6 to 20 seconds and are grouped by family: jump off, jump on, stairs down pure, stairs up pure, stairs down transition, stairs up transition, and flat. To generate these motion clips, velocity profiles are sampled and for each steady state velocity command, the closest steady state optimized reference is chosen. Then Motion Bricks \cite{wang_motionbricks_2026} is used to do motion in-betweening. This does the same work that something like motion matching \cite{wu_perceptive_2026} could do, but much faster and smoother.

10,000 motion clips are generated over 140 different tile geometries to create a single motion library. This allows the training to see many different velocity profiles over the same geometry and many different geometries.

\subsection{Tracker}
To create the tracker we train an RL policy that tracks the pre-made motions generated above. To accomplish this we utilize a variant of CLF-RL \cite{li_clf-rl_2026}, which we choose because it has shown to provide lower tracking error \cite{olkin_chasing_2026}. This variant of CLF-RL tracks body positions for all the links instead of the joint angles. We also introduce an additional scaling term to help balance the numerical properties without sacrificing the CLF properties (see the appendix for more details). The tracker rewards include the two CLF-RL tracking rewards, an undesired contact penalty, torque limit penalty, joint limit penalty, action rate penalty, and a torque penalty.

This is a perceptive policy and therefore takes in a depth camera scan which has self scanning. This scan is only 30x26 pixels and on hardware is constructed by down-sampling the 600p image. This follows a similar vein to \cite{compton_terrain_2026} which also uses raw depth scans that see its own links. The full observations are given in Table \ref{tab:observations}. The network is a feed forward MLP that encodes the reference and commanded velocity information which then feeds into another MLP which also takes the proprioceptive and depth inputs all with ELU activations. The reference and command MLP has hidden dimensions [1024, 512, 512] and the full MLP has hidden dimensions [1024, 512, 256, 128].

The policy runs at 50 Hz and outputs joint position targets that go into joint level PD controllers. The action rates and gains are derived from \cite{liao_beyondmimic_2025}. Domain randomization on physical values such as the center of mass, joint friction, ground restitution, and camera angles and offsets is used to ease sim-to-real transfer. Further, noise on the reference trajectory, depth camera readings, and proprioceptive inputs is also used. The RL tracking controller was trained in IsaacLab \cite{nvidia_isaac_2025} using 8192 environments on a single H100 for 20,000 iterations, lasting approximately 48 hours. PPO with asymmetric actor critic is used to train this policy with the implementation from RSL-RL \cite{schwarke_rsl-rl_2025}.

\begin{table}[]
    \centering
    \begin{tabular}{c|c|c}
    \toprule
        Term & Size & History \\
        \hline
        Joint angles & 29 & 10x \\
        Joint velocities & 29 & 10x \\
        Projected gravity & 3 & 10x \\
        Root angular velocity & 3 & 10x \\
        Previous action & 29 & 10x \\
        Lower depth camera & 780 & N/A\\
        Reference trajectory & 572 & N/A\\
        Velocity command & 39 & N/A\\
        \bottomrule
    \end{tabular}
    \caption{Tracker observation terms and their sizes. The terms with a history are stacked over that many steps.}
    \label{tab:observations}
    \vspace{-4mm}
\end{table}

\subsection{Generator}
The generator is a flow matching transformer that is auto-regressive in nature with the only sensor observation being the depth images. The depth images go through a CNN into tokens that feed into the transformer. The transformer also gets the velocity command and a trajectory history in the form of a subset of nodes of the previously generated trajectory. Each node of the history becomes a token. The flow matching head outputs the reference trajectory.

The transformer is structured so each of the CNN outputs per camera get their own token (56 per camera/112 total perceptive tokens). Each node in the trajectory history gets its own token, and the velocity commands are split across 13 tokens. There is full self attention on all of the conditioning/``prefix" tokens, but these tokens do not attend to the output trajectory tokens in the flow matching head. This means that the internal representation of the depth images, velocity command, and history are independent of the current flow matching velocity field output. This allows the transformer encoder to be used later in the RL fine tuning, and it also speeds up inference time as the prefix encoder only needs to be computed once then it can be cached and re-used at each integration step. The output tokens attend to the conditioning tokens. The perceptive tokens get positional embedding based on their angle in the camera field of view.

The inputs to the CNN from the depth camera have five channels: depth, $x$, $y$, and $z$ positions, and a valid binary. The depth is just the raw depth from the down sampled image. The three position coordinates are the positions of that pixel in 3D space relative to a frame at the root link of the robot in gravity aligned space and aligned with the yaw of the robot. So as the robot yaws, this frame moves but as it pitches or rolls the frame does not move. These coordinates are computed via forward kinematics to get the transform from the root link to the camera pose and then transforming the position of the depth pixel into the root frame. This gravity aligned root frame is the frame that the flow matching produces the root positions in. Given the lack of proprioception entering into the flow matching policy we found that these additional coordinate inputs were critical for getting the generator to work on terrain. Without using these position coordinates the robot cannot discern certain situations, such as if the robot is leaning towards terrain or if terrain is getting closer without the robot moving.

To collect data for the training we use the trained tracker policy and have it follow the pre-made motion clips it is trained on. We record the depth images and the reference trajectory before (for the history) and after that point (for the predicted trajectory). We choose to use the actual RL rollout to make the visited states of the depth camera as realistic as possible. Additionally, we hold a buffer of previous depth camera measurements so that we can train the generator to account for lag in the depth camera measurements by using an image out of the buffer instead of the most current image. After we have collected an initial dataset size of 200,000 points we now augment the data. We do this data augmentation to emphasize in the data how the reference must be terrain consistent. To do this augmentation, all the data points that have a trajectory that is fully or partially on terrain are used. At these data points we sample a random distance along the terrain and with a yaw offset and apply this to the state of the robot, then the depth image is re-casted from this spot and the reference stays the same. This shows the training that these axes must stay correct relative to the terrain even as the robot moves. We do 3 augmentation points per relevant data point.

During training, 20\% of the data points have their history replaced with a learned null embedding. This does two things: (1) it forces the policy to not just continue what the history was doing and instead it must learn the mapping from the velocity command and depth image to the output trajectory and (2) it gives us a learned parameter to use as the history for the first inference when there is no history to use yet. Also, in training we add noise to the history trajectory to make the policy more robust to its own outputs when it runs auto-regressively online. We find that the final trained policy is very robust and stable in its auto-regressive rollouts even though it is never trained to see its own outputs, it is only ever teacher forced.

The flow matching generator is trained for 100 epochs on an H100 GPU taking about 12 hours to complete. When we deploy the generator we use eight Euler integration steps. The model has eight layers with each layer having eight heads. The feed forward networks are size 2048 and $d_{\text{model}}$ is size 512. There are 26.8 M total learnable parameters.

\begin{figure}
    \centering
    \includegraphics[width=1.0\linewidth]{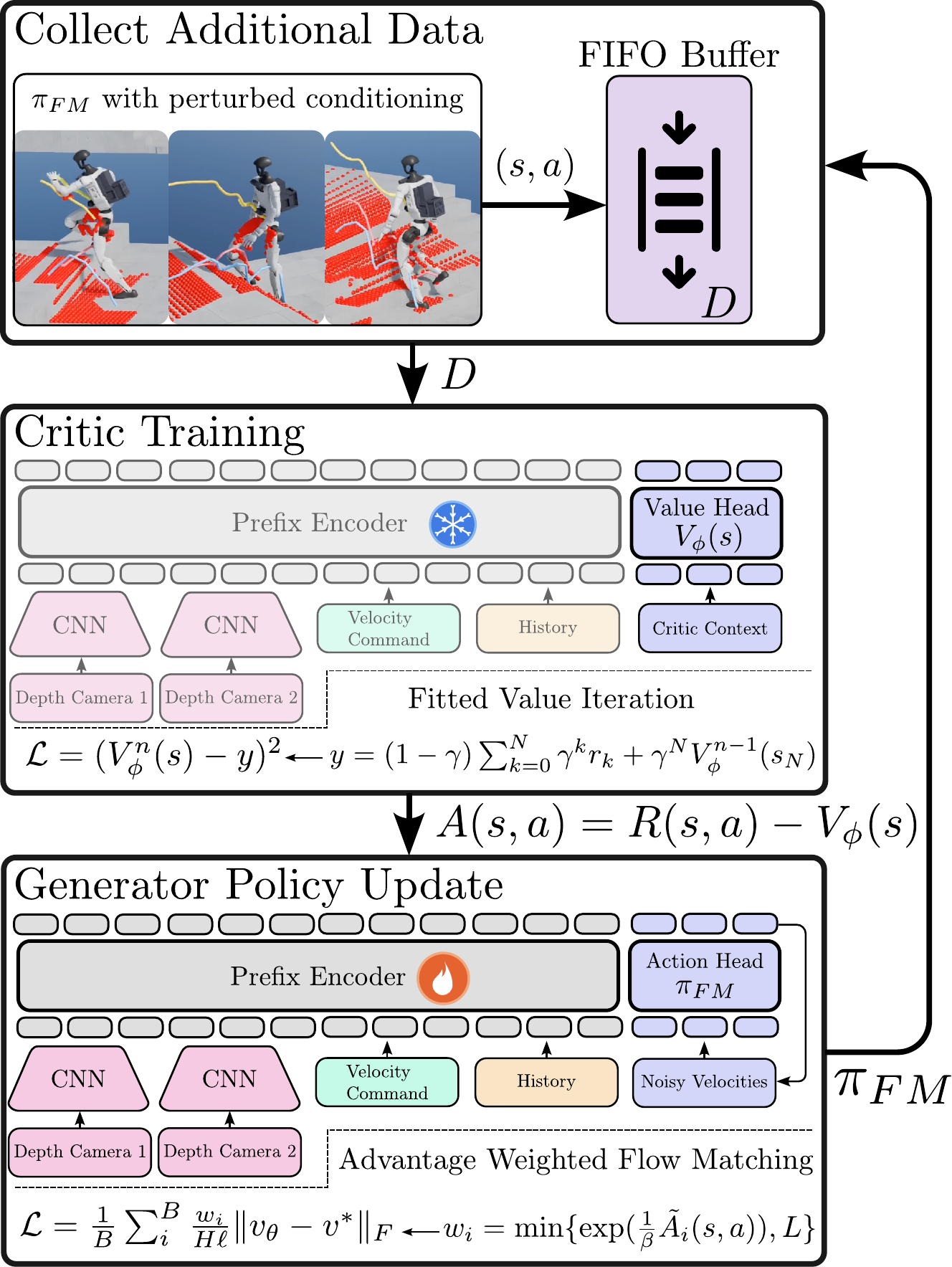}
    \caption{Generator RL fine tuning architecture. The overall process is depicted here: collect data, train the critic, and update the generator. The transformer architectures are also depicted here, showing the prefix encoder and how it is frozen in critic training and updated in the generator. A value head is used for the critic while a velocity field action head is used for the generator.}
    \label{fig:transformer_architectures}
    \vspace{-4mm}
\end{figure}

\subsection{RL Fine Tuning for the Generator}
After the initial training, the flow matching policy can be run in the loop with the tracker, yet it still struggles with out of distribution terrains and task compositions. To address these concerns we propose a simple yet effective method for RL fine tuning of the motion generator, not the tracker, as shown in Fig. \ref{fig:transformer_architectures}. The motion generator is formulated as a flow matching transformer that outputs almost 3,000 values (a whole body trajectory) each re-plan. Due to the fact that this is a flow matching model and the high dimensional output space we choose to use an RL algorithm that relies on supervised learning and where exploration doesn't enter through gaussian noise on the outputs. As we show in the results below, doing something like PPO on a residual policy is very data inefficient and difficult to explore with.

We utilize advantage labeling for advantage weighted regression (AWR) \cite{peng_advantage-weighted_2019} as an off policy RL algorithm that utilizes supervised learning. Therefore our algorithm has three main parts: (1) data collection, (2) critic learning and advantage labeling and (3) supervised learning of the flow matching policy. Importantly, rather than inserting gaussian noise on all the action outputs, which does not lead to efficient search for a 3,000 dimensional action space, we instead perturb the conditioning to allow the robot to see different possible actions, that are still reasonable, as well as spawning the robot into a variety of positions that it may not have achieved with the original policy, and using the noise on the initial sample in the flow matching. This lets us ``search" in a more structured manner. When we perturb the conditioning we still store the nominal conditioning so that the policy training sees that as a valid action at the state it was actually at. Since this is an off policy algorithm we don't need to collect too much data which is helpful as rollouts now need to simulate the physics, query the tracking policy, and run inference on the flow matching transformer for eight integration steps. This makes rollouts more expensive to collect compared to the tracker PPO learning. 

To give more details, we start by describing the RL formulation for the generator fine tuning problem at hand. Consider states $s_k \in \mathcal{S}$ and actions $a_k \in \mathcal{A}$ where $k$ denotes the re-plan step. Note that the steps in this RL environment are re-planning steps as the tracking policy is frozen. So from the perspective of this MDP, the tracker is effectively part of the dynamics of the robot and not something that can be adjusted. The state is all of the conditioning variables: the depth scan, reference history and velocity command, while the actions are the full trajectory and therefore $\mathcal{A} \subset \mathbb{R}^{H\times \ell}$ where $H$ are the number of nodes in the trajectory and $\ell$ is the dimension of a single node, so in our case $H = 62$ and $\ell = 44$ so that $H \times \ell = 2728$. The flow matching policy produces a sampled action given a state: $a_k \sim \pi_{FM}(a | s_k)$. Then we can define a step wise reward $r_k = r(s_k, a_k)$ as a function of the states and actions and define a sequence of state action pairs as $\tau = \{ (s_0, a_0), (s_1, a_1), .... \}$ which implicitly defines a corresponding set of rewards $r_k$.  The infinite horizon discounted reward is
\begin{equation}
    J_{\pi_{FM}}(s) = \sum_{k=0}^\infty \gamma^k r_k.
\end{equation}
The specific rewards used are given in Table \ref{tab:generator_rl_rewards}. These rewards encourage terrain consistency first and foremost. The terrain penetration penalizes bodies that intersect terrain and the terrain contact term penalizes bodies that should be in contact (as determined by a kinematic classifier) that are either too far into terrain or too far away from it. These rewards are inspired by the rewards from \cite{xu_parc_2025}. Additionally, we have velocity tracking and success rewards that are modulated by terrain consistency. Velocity tracking and success are therefore secondary rewards that only provide significant rewards when the terrain consistency is good. Given that deviating from the commanded velocity may be desired when traversing terrain, this allows us to reward both in a consistent way. The success is gated on the worst terrain consistency over the episode and therefore only provides a strong reward signal when every plan is terrain consistent.

\begin{table*}[]
    \centering
    \begin{tabular}{c|c}
        Term & Value \\
        \hline
        \hline
        Terrain Penetration ($r_{pen}$) & $\sum_{b=1}^{N_b} w_b \sum_{p \in P(b)} -\min \{\text{sdf}(p), 0\}$ \\
        Terrain Contact ($r_{con}$) & $\sum_{b = 1}^{N_b} c_b \min_{p \in P(b)} |\text{sdf}(p)|$ \\
        Terrain Consistency ($r_{terr}$) & $\exp(-(w_c r_{con} + w_p r_{pen}))$\\
        \hline
        Velocity Tracking ($r_{vel}$) & $\sin(\frac{\pi}{2} (r_{terr})^{s_{vel}}) \exp(-|v_x - v^{cmd}_x| / \sigma)$ \\
        Success ($r_{succ})$ & $\sin(\frac{\pi}{2} (\min_k{r_{terr}})^{s_{succ}}) \cdot C 1_{reached} $ \\
        \hline
        Total & $r_{terr} + r_{vel} + r_{succ}$
        
    \end{tabular}
    \caption{Rewards for the generator RL. These rewards encourage terrain consistency then velocity tracking and success secondarily. $N_b$ denotes the number of body links and $P(b)$ is the set of points on a given body. $\text{sdf}(p)$ denotes the distance using a signed distance field from terrain to that point. $s_{vel}$ and $s_{succ}$ are tuning terms to affect how sharply the terrain terms fall off.}
    \label{tab:generator_rl_rewards}
    \vspace{-4mm}
\end{table*}

To start training, we need to collect data. We do this by running the robot with the fixed motion tracker in IsaacLab where the flow matching policy is driving it. Critically, this allows us to abandon the pre-made motion clips at this point and fine tune on any terrain. We create a FIFO data buffer that holds $D$ data points. Each data point is a pair $(s, a)$ as described above. At the start we collect 2,000 episodes worth of data points (re-plans). Then each subsequent iteration we only collect 500 more episodes worth of data and insert them into the buffer, ejecting 500 episodes worth of the oldest data points.

Our method relies on utilizing advantage estimates to understand which state actions pairs are better. Therefore we fit a critic network to predict the infinite horizon discounted reward $J_{\pi_{FM}}(s)$. We train the critic through fitted value iteration (FVI) and we utilize the frozen prefix encoder from the original generator and thus only train a value head. For the first iteration of FVI we estimate the rewards through a Monte Carlo estimate with no bootstrapping: $V(s_i) \approx \sum_{k = i}^\infty r_k$ with terminations that truncate the rollout. If the rollout terminates due to a failure at step $k$ then all future rewards are set to 0. Otherwise if the robot doesn't fall then we still truncate the rollout but bootstrap a tail value that is just the reward of the last step. For all iterations of FVI at $n > 0$ we use a $N$-step lookahead and bootstrap with the previous value of $V$\footnote[2]{We regress the $(1-\gamma)$ scaled value so targets are in the reward scale.}:
\begin{equation}
    V^n(s) = (1 - \gamma) \sum_{k=0}^{N-1} \gamma^{k} r_k + \gamma^N V^{n-1}(s_{N}).
    \label{eq:fvi_update}
\end{equation}

After running FVI we use the critic for advantage labeling. The advantage is $A(s, a) = R(s,a) - V(s)$ where $R(s,a)$ is computed as the sampled return from a state. No critic is used in the computation of $R(s,a)$ as it unrolls the trajectory to termination with no bootstrapping so the true return is always used for advantage labeling. In practice we choose to use a normalized advantage $\tilde{A} = \frac{A}{\sigma_{A}}$ where $\sigma_A$ is the standard deviation of $A$ computed once per RL iteration. Then the advantage weight $\beta$, used below, can we written in terms of normalized units. We train two critics, each on half the data, and use them to label advantages on the other portion of the data to prevent potential issues with memorizing episodes given there are multiple data points with the same success or failure label in a given episode. 

Once we have the advantages we can label each data point and then weight the importance of each data point by its advantage: $w = \min\{\exp(\frac{1}{\beta}\tilde{A}(s,a)), L\}$ where $L$ is an upper bound on the weight. Then we use these weights for the flow matching MSE loss by uniformly sampling a time $t \in [0, 1]$ and sampling an initial noise draw $x_0 \sim \mathcal{N}(0, I)$ and interpolating to get $x_t = (1 - t)x_0 + tx_1$. The linear velocity field is given by $v^* = x_1 - x_0$ and the weighted flow matching loss is:
\begin{equation}
    \mathcal{L} = \frac{1}{B} \sum_i^B \frac{w_i}{H\ell}\|v_{\theta}(x_t^i, t_i, s_i) - v^*\|^2_F
    \label{eq:weighted_flow_loss}
\end{equation}
where $B$ is the batch size, and $\|\cdot\|_F$ is the Frobenius norm.

Additional data can be inserted into this process easily. We can choose to use data from the pre-training in the fine tuning by adding it into the flow matching update with weight one and not using it in the critic learning or advantage labeling. Therefore we can easily mitigate any potential issues with forgetting previous behaviors.

We do RL fine tuning of the generator for 5 iterations. In each iteration five epochs are used in the flow matching update. Training occurs on a single H100 GPU and lasts about six hours (although this can vary depending on the terrain, episode length, and other factors).

\subsection{Deployment}
To deploy the policies on the hardware we choose to run them on an NVIDIA Jetson Thor. Flow matching inference time is around 11 ms while the tracker is less than 1 ms. Although the generator is only queried once every 0.24 s, mitigating the delay between sensor input and new trajectory is still critical for dynamic motions. The two cameras are a Zed X mini and Zed X which both feed into a NVIDIA Jetson Orin for depth processing and down sampling before being sent over ROS2 to the Thor. In training we only simulate the down sampled pixels instead of simulating the whole camera and down sampling there. The Zed X sits towards the top of the torso and points mostly forward while the Zed X mini is on the lower part of the torso and sees the legs and the ground immediately under the robot. The Thor and Orin are both powered from the robot and a joystick is connected over Bluetooth for providing velocity commands.

\section{Results}
Given this architecture, we can deploy the policies on real robots and across simulations to study its performance. The effects of fine tuning, choice of algorithm, the velocity tracking and terrain traversal, using two cameras, and the data pipeline are studied quantitatively. Additionally, hardware experiments on the Unitree G1 are run to demonstrate the effectiveness of the policy on a real robot outdoors and indoors in real world environments.

\subsection{RL Fine Tuning}
\begin{figure*}
    \centering
    \includegraphics[width=1.0\linewidth]{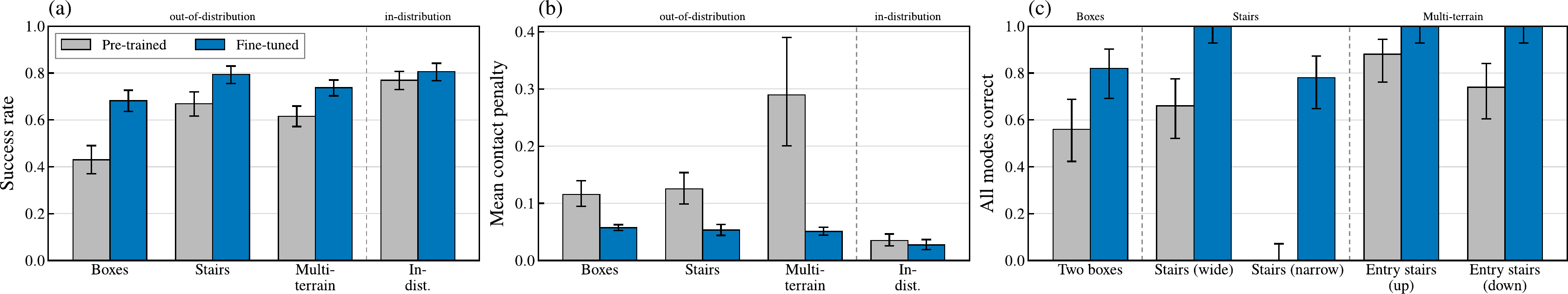}
    \caption{Effects of using the RL fine tuning for both out of distribution and in distribution terrains. The out of distribution terrains vary in both the task being accomplished and the terrain. The success rate is shown in (a) where a success is counted as having reached the goal area without falling and within the given time. The goal regions are across the terrain. Plot (b) shows the mean contact penalty so a lower number is better, and therefore the fine tuned policies are more terrain consistent than the pre-trained policies. Finally (c) shows what fraction of sampled episodes get every mode (skill) correct. For example, for the two box terrain this includes two jump ups and two jump downs. The fine tuned policies drastically improve mode selection: on one stairs terrain it increases from 0\% correct to about 80\% correct. 95\% confidence intervals are given in each plot as error bars.}
    \label{fig:fine_tuning_improvements}
    \vspace{-5mm}
\end{figure*}

The generator RL fine tuning allows us to improve and adapt the generator to out of distribution environments or tasks, such as having additional walls in view, or task compositions like jumping on then off of a box. To test the effects of the RL fine tuning we created 4 distinct out of distribution fine tuning configurations: (1) boxes, which has 1 - 2 boxes of varying widths and heights, (2) stairs that go up then down with varying widths, 1-2 walls flanking it, a start and end platform, and a central platform between an up and down set of stairs, (3) a multi-terrain setup that includes terrains from the boxes, stairs, some in distribution terrains, and a realistic building entry way that includes two stair cases of varying dimensions, walls, and a landing between the stair cases, and (4) an in-distribution set of terrain that were used for the original generator and tracker training. The first three of these terrains are out-of-distribution for a number of reasons including that the policy was never trained to do a box jump up and down in a single episode, box widths were varied, walls were added around the stairs, and entirely new terrains, like the entry way, were added. 

Fig. \ref{fig:fine_tuning_improvements} shows the improvements gained from the RL fine tuning. The success rates increased across the board by as much as 25 percentage points while the mean contact penalty decreased, showing better agreement between the terrain and the trajectory, and finally, the choice of ``mode" significantly improved for each terrain type too. The mode here is referring to the skill used to traverse the terrain, e.g. using a stairs trajectory or box jump trajectory. The choice of mode is determined by rolling out 50 samples per policy and per terrain. Then the rollouts were randomly shuffled together to prevent bias and a human manually labeled each of the modes chosen, and if the human could not determine the type of mode taken it was counted as incorrect. This demonstrates the capabilities of the generator RL to shift the output distribution to the preferred mode, increase success rates, and increase terrain consistency. Adjusting this choice of mode is not something that can be achieved with fine tuning the tracker policy, making this a complementary tool. 

The multi-terrain fine tuning was the policy that was deployed on hardware. The multi-terrain fine tuning shows the ability of the fine tuning to work across terrain types and skills. This fine tuning also included data from the original pre-training as discussed above. This allows us to both train on new environments while also feeding in the original data making sure the generator doesn't forget other motions.

\subsection{Algorithm Comparison}
\begin{figure}
    \centering
    \includegraphics[width=1.0\linewidth]{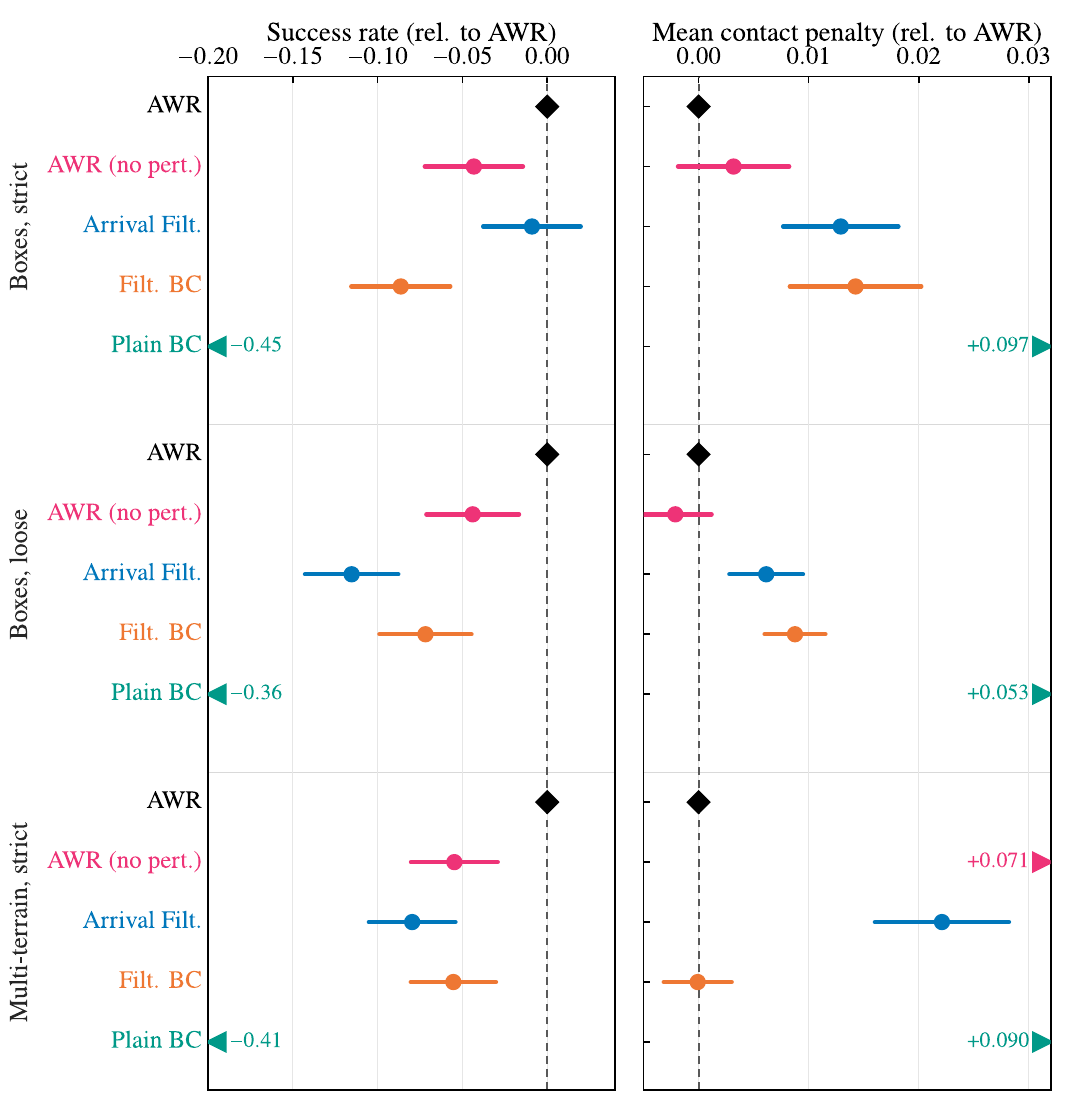}
    \caption{Comparison between different fine tuning algorithms and search methods for the generator relative to using AWR as described above. The methods are compared across terrains and across a strict and loose termination for the boxes. The AWR algorithm performs equivalently or better than each other algorithm in both metrics. The 95\% confidence intervals are given in the plot.}
    \label{fig:finetuning_algo_comp}
    \vspace{-3mm}
\end{figure}

Having shown the importance of fine tuning the pre-trained policy we now study our algorithmic choices. Our method updates the policy via supervised regression on off-policy rollouts using advantage weighting (AWR). We compare to four other methods: filtered behavior cloning (filt. BC) \cite{emmons_rvs_2022}, arrival filtering (keeping all the data in rollouts that reach the goal region), plain BC/self distillation, and AWR without the perturbed conditioning (AWR no pert.). The plain BC removes the advantage information entirely and asks if having any data on this new terrain is helpful even in the absence of a critic or selection mechanism. The arrival filtering and filtered BC investigate how the choice of selection method affects the results. The removal of the perturbed conditioning isolates the effect of that specific search methodology.

Fig. \ref{fig:finetuning_algo_comp} shows that AWR consistently matches or out performs all the other algorithmic choices both in success rate and terrain consistency. Plain BC fails badly, demonstrating that some form of advantage or selection mechanism is critical. The other two selection variants both generally perform worse than AWR in both success rate and terrain consistency. We find that the arrival filtering is much more sensitive to the choice of termination. Switching from the strict termination (i.e. terminating more than 12 cm from the generated trajectory) to the loose termination (30 cm termination bound) we find that the arrival filtering loses about 10 percentage points of its success rate relative to AWR while the filtered BC is consistently 5-10 percentage points behind AWR. By having terminations based on tracking performance the arrival filtering can implicitly improve the terrain consistency metric, which is why it is not as bad as plain BC, but it still cannot match the AWR performance. Removing the conditioning perturbations reduces success by 4-6 percentage points. The contact penalty is also much higher on the multi-terrain with consistently worse performance on stair terrains. Overall, the learned advantage function, the advantage-weighted updates, and the conditioning perturbations each contribute to performance, consistently improving success rates and, in most settings, terrain consistency, which motivates our use of this policy on hardware.

We further compare with using PPO to train a residual policy that takes as observation the generator output and its observations, and outputs a residual adjustment to the trajectory. This residual policy has an action space the size of the full generator output. If the residual worked in theory this could be used to distill back into the original generator. This lets us study how our algorithm compares to a residual fine tuning method using PPO, which is on-policy. The PPO baseline uses the same rewards as the off policy algorithm. We find that the PPO method is significantly less data efficient: we stopped training after 2,000 iterations at which point it had used more than 160x the data of the off policy algorithm and took more than 30x the wall clock time. At this point, the PPO baseline achieves a success rate around 13 percentage points below the AWR's. Although it may still continue to trend upwards, it would require substantially more data and compute time. Therefore we conclude that an off policy algorithm like AWR is the better approach for RL fine tuning in this context.

\subsection{Velocity Tracking and Terrain Traversal}
\begin{figure}
    \centering
    \includegraphics[width=1.0\linewidth]{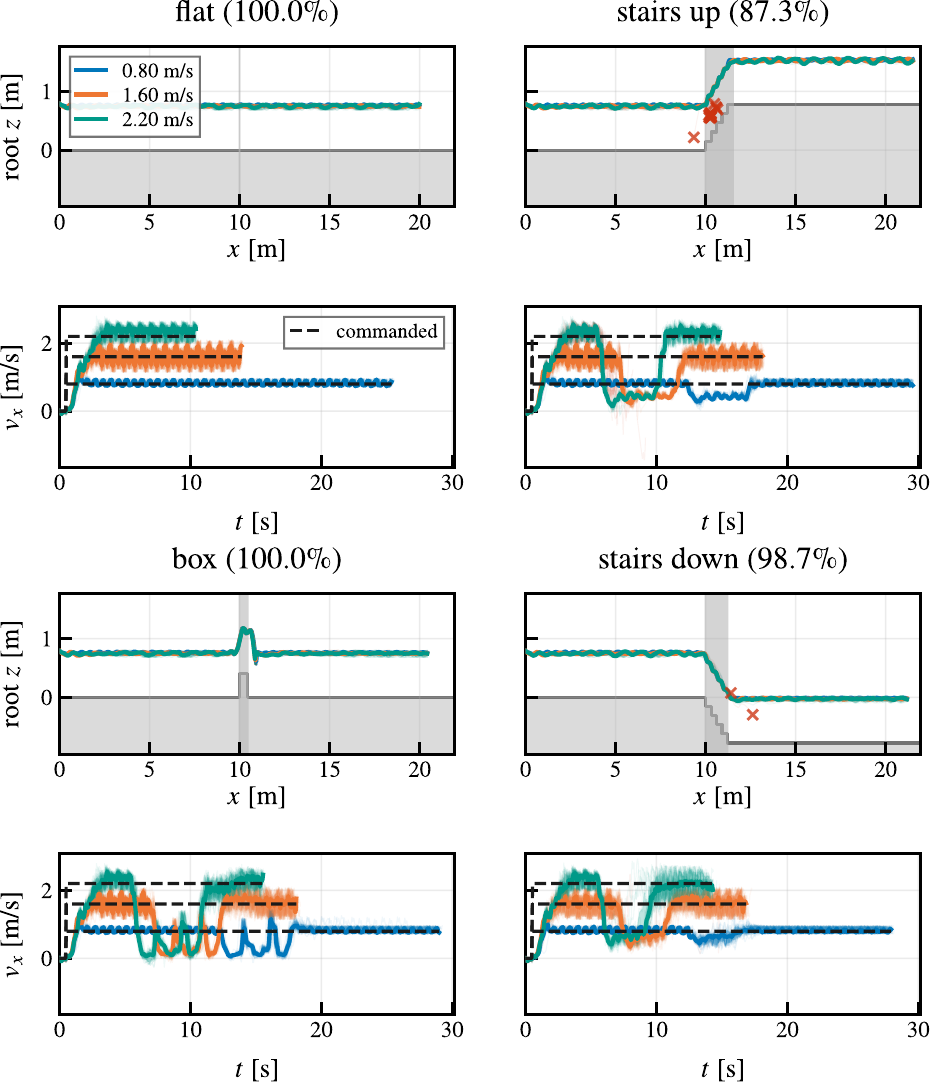}
    \caption{Demonstration of the terrain traversal capabilities of the policies and their effect on the robot's velocities when using a fine tuned policy. The robot autonomously slows down upon approaching terrain so that it can robustly traverse it. In general the policy shows accurate speed tracking and reliable terrain traversal. The percentages are measuring the successful trials over the 150 samples per terrain.}
    \label{fig:vel_changes}
    \vspace{-5mm}
\end{figure}

Using the fine tuned policy we can see how well the robot tracks velocities before, after, and during terrain traversals. It should be noted that the terrain traversals cannot happen at any speed. We don't have a stairs sprinting trajectory and we don't have the ability to jump onto a box at fast speeds either. Therefore, it is expected and desired for the robot to slow itself down. Yet, having to command the correct velocity to traverse an obstacle is undesired. Therefore the policy can use its depth cameras to autonomously adjust its speed to meet the needs of the terrain. Fig. \ref{fig:vel_changes} demonstrates how the policy pair leads to accurate velocity tracking at a variety of speeds and how it can slow down to the necessary speed to traverse an obstacle and then go back to the desired speed. This figure also shows how the fine tuned policy does not lose its ability to track velocities, even with the perturbed conditioning during training. 50 rollouts at each speed in MuJoCo were used to evaluate the policy each with a different random seed for the flow matching sampling. Therefore this also validates the sim-to-sim performance of the policy.

\subsection{Hardware Deployment}
\begin{figure*}
    \centering
    \includegraphics[width=1.0\linewidth]{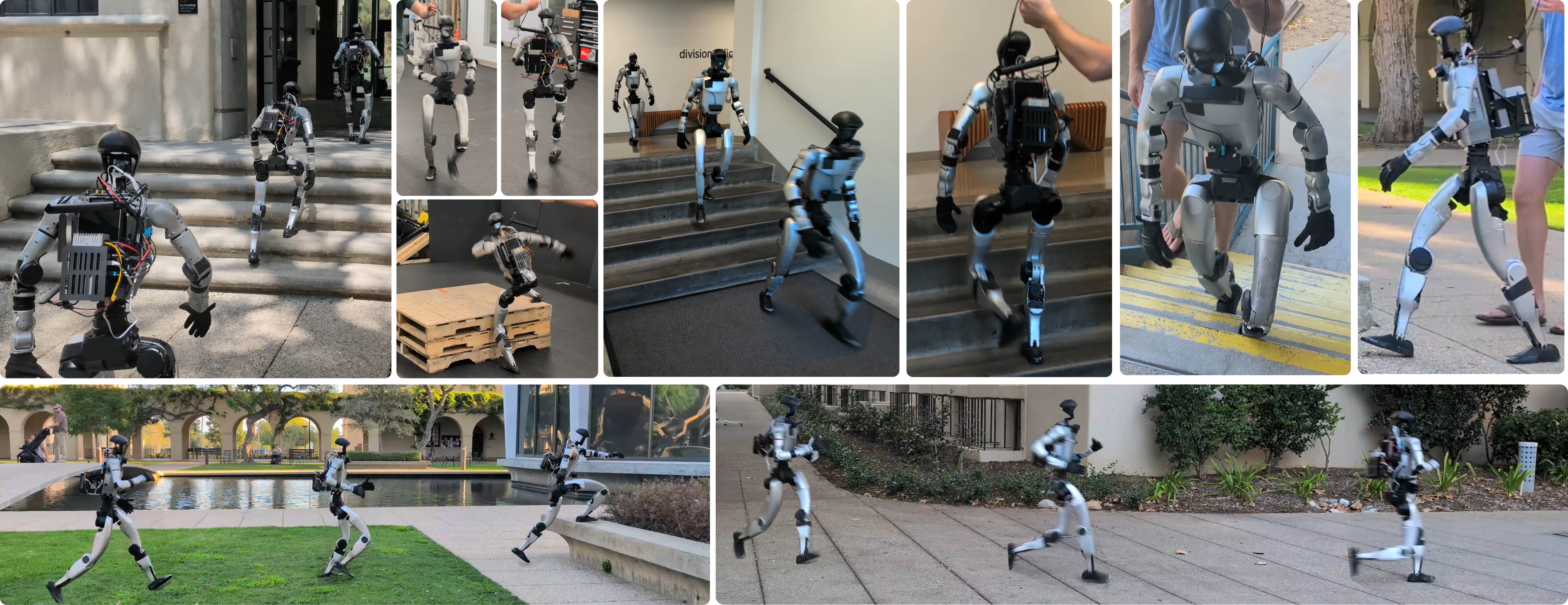}
    \caption{Demonstration of the policy pair working in a variety of conditions. We demonstrate traversing a number of different stair cases that have unique geometry indicating the policy's capability to adjust to different terrain. Then we also show running and walking of the robot on flat ground both in straight lines and while turning. We also demonstrate the robot jumping onto a variety of platforms including from a running start.}
    \label{fig:hardware_lower}
    \vspace{-4mm}
\end{figure*}

Figs. \ref{fig:hero} and \ref{fig:hardware_lower} demonstrate the robot working in both outdoor and indoor settings. By using the raw depth camera readings the transfer to the outdoors is easy and no adjustments are needed relative to the inside. We see the robot climb 15 consecutive steps without issue demonstrating the robustness and precision of the policies. The robot can locomote over many different real-world stair cases: 6 different real-world stair cases were successfully traversed, each with different geometry. The robot can locomote both down and up stairs. The robot can walk and run both indoors and outdoors as well as transition from fast and slow speeds onto the terrain. In Fig. \ref{fig:hardware_lower} the bottom right image demonstrates the robot running at 2 m/s while turning on a path. Lastly the robot can jump onto and off of boxes. We have tested approaching stairs at a variety of speeds on hardware including at 0.8 m/s, 1.5 m/s and 2.0 m/s while successfully transitioning onto the terrain. On the hardware, the resulting motions are smooth, natural, and relatively quiet, especially on terrain like stairs where other methods may slam their feet into the ground. The hardware experiments also validate that the multi-camera policy can be deployed in real world environments with walls and trees and bushes without needing to simulate every possible perceptive condition. Our results demonstrate that this policy can be deployed to control a humanoid robot in real world environments for dynamic and terrain aware multi-skill locomotion.

\subsection{Data Pipeline}
\begin{table}[t]
    \centering
{\scriptsize
    \begin{tabular}{c|c|c|c|c|c|c}
    \toprule
    & \multicolumn{2}{c|}{Velocity RMSE (m/s)}\\
Vel. Command ($v_x$) & Ours [95\% CI] & Motion Bricks [95\% CI]\\
\hline
  (-1.0, -0.2]  &   \underline{0.484 [0.357, 0.598]} &  \underline{0.462 [0.346, 0.562]}\\
  (-0.2, +1.0]  &   \bf{0.230 [0.191, 0.271]} &  0.368 [0.308, 0.434] \\
  (+1.0, +2.5]  &   \bf{0.442 [0.312, 0.563]} &  0.896 [0.813, 0.976] \\
  \bottomrule
    \end{tabular}
    }
    \caption{Comparison between different motion clip generation methods on flat ground. Entries that are statistically significant improvements are bold and entries that have CI bounds overlapping other means are underlined.}
    \label{tab:vel_tracking_clips}
    \vspace{-4mm}
\end{table}
Next, we can study the properties of our data pipeline that feeds into the tracker and the generator. Table \ref{tab:vel_tracking_clips} shows that our optimized reference trajectories in conjunction with Motion Bricks achieve superior velocity tracking to that of Motion Bricks alone. Achieving accurate velocity tracking is critical as these motion clips upper bound the velocity tracking performance of the final policies since they are trained to track these clips. Therefore any lost tracking capabilities here can not be recovered with the current pipeline. Although just using Motion Bricks would certainly be the lower effort option, it would lead to significantly worse tracking performance. We also investigated the use of motion matching for the generation of offline clips and found that it was able to achieve good velocity tracking when given access to the same reference motions and dataset but the motion matching took more than 7x the time to compute a given clip, which translated to hours for the full 10,000 motion clip library, making it not a desirable path forward. Additionally, the motion matching produced higher jerk trajectories resulting in motions that appeared to be more jittery.

\subsection{Two Camera Ablation}
 \begin{table}[]
    \centering
    \begin{tabular}{c|c|c}
    \toprule
        Terrain Type & Both Cameras & Lower Only  \\
        \hline
        Flat & 100\% & 100\% \\
        Box & \textbf{100 \%} & 57.3 \% \\
        Stairs Up & \textbf{82.7\%} & 71.3\%  \\
        Stairs Down & \textbf{97.3\%} & 88.0 \% \\
    \bottomrule
    \end{tabular}
    \caption{Comparison showing how the upper camera is affecting the ability to traverse in-distribution terrain. Removing the upper camera hurts the ability to traverse a given terrain by up to 43\% as demonstrated with the box terrain. The upper camera is critical for seeing the object ahead of time so the policy can modulate its speed early.}
    \label{tab:camera_ablation}
    \vspace{-5mm}
\end{table}

We study the effect of having two cameras given that many other policies use only a single camera. The upper camera is pointed forward to allow the robot to see what is coming while the lower camera is pointed down to see the terrain immediately under its feet. Previous results have used only the downward facing camera \cite{compton_terrain_2026, zhang_rpl_2026} but our work requires the robot to adjust its velocity before it reaches the terrain since the robot can travel at faster speeds (up to 2.5 m/s). We demonstrate that the upper camera is a critical part of the architecture to enable this velocity adjustment by training a generator with and without the upper camera. In this experiment we choose to only use the original pre-trained generator to reduce the number of confounding variables. Table \ref{tab:camera_ablation} shows that there is a significant decrease in capabilities with success rates dropping by up to 43\% with the upper camera removed. Looking at the trajectories for the box terrain confirms that the vast majority of these failures are caused by running into the box, and the success rates drop with increased speed.

\section{Conclusion}
\label{sec:conclusion}
We presented Generate, Track, Improve: a full pipeline for human data optimization, motion clip creation, tracker training, motion generation, and finally an off policy RL loop for the generator. The resulting policy pair enables a humanoid robot to achieve multiple skills such as walking, running, standing, jumping on and off of boxes, and traversing stairs, all in the real world with perception in the loop. The policies take in raw depth images from multiple cameras which allow the robot to see the current obstacles and what is coming. The RL fine tuning loop with AWR improves the mode selection by up to 80 percentage points, and success rate by up to 25 percentage points. This allows the policy to adapt to out of distribution tasks and geometries using a simple but effective method. 

The policies were demonstrated on real humanoid robots both in lab settings and outdoors. The policies create motions that are human like, dynamic, and robust enough for deployment. The robot is able to traverse long real world stair cases, jump on, walk across, and jump off of boxes, walk and run. The resulting policies showed a few different skills but can be easily extended to many more skills, making this architecture amenable to general purpose humanoids. 

\subsection{Limitations}
Hand designed rewards are used in the RL fine tuning, and although they are quite general, may not extend to every possible motion. Using only depth information limits the semantic information that a policy like this can use. For example, it is possible that there are box shaped objects we do not want to jump on and currently we have no way to encode that. The data pipeline still required some hand tuning of transition hyperparameters and choosing the human data reference motions for the optimization is still done via either a heuristic or a human. Both of these data pipeline limitations need to be addressed before the method can truly scale to any motion. 

\subsection{Future Work}
Improving the search in the generator RL could not only improve the quality of the motions but it could also lead to new motions being generated, similar to the work in PARC \cite{xu_parc_2025}. But finding an efficient and structured way to search will be critical. Our proposed method could also be extended to learn from real-world hardware data as well, allowing it to improve from real deployments. To generate terrain aware motions using a mimic style pipeline at scale we generally use reference motions from human data, but, the data we use does not have ground truth terrain information right now and including objects and terrain in the dataset could make the problem much more scalable. Developing better methods for transitions and new models to generate terrain aware motions offline will be critical here. Additionally, folding these controllers into a navigation autonomy stack and threading in safety/obstacle avoidance will be necessary for further real world deployment.

\bibliographystyle{IEEEtran}
\bibliography{IEEEabrv, generator_rl}

\section{Appendix}
\subsection{CLF-RL}
We use CLF-RL \cite{li_clf-rl_2026} for our motion tracking rewards. A CLF nominally up-weights tracking of velocity terms but we found in practice that since the velocity can be higher variance than the position this is not always desirable. Therefore we choose to insert a scaling term $S$ that re-scales the CLF weights without violating its theoretical properties.

\begin{proposition}
Consider the outputs $\eta$ ordered so that $A_\eta = \mathrm{blkdiag}(A_1,\dots,A_n)$, $B_\eta = \mathrm{blkdiag}(B_1,\dots,B_n)$ with
$A_i = \begin{bmatrix}0&1\\0&0\end{bmatrix}$, $B_i = \begin{bmatrix}0\\1\end{bmatrix}$.
Further, let $Q$ and $R$ be diagonal and positive definite, and let $P$ be the stabilizing solution of the CARE for
$(A_\eta, B_\eta, Q, R)$. Let $S = \mathrm{diag}(a_1, b_1, \dots, a_n, b_n)$ with
$a_i, b_i > 0$. Under the assumptions of Prop. 1 of \cite{li_clf-rl_2026},
$\bar V(\eta) = \eta^\top S P S \eta$ is an exponentially stabilizing CLF.
\end{proposition}

\begin{proof}
Since $A_\eta$, $B_\eta$, $Q$, $R$ are block diagonal by output, the CARE decouples and
$P = \mathrm{blkdiag}(P_1,\dots,P_n)$, where $P_i$ solves
\begin{equation}
  A_i^\top P_i + P_i A_i - P_i B_i r_i^{-1} B_i^\top P_i + Q_i = 0,
  \label{eq:care_i}
\end{equation}
Note $SPS = \mathrm{blkdiag}(S_i P_i S_i)$ with $S_i = \mathrm{diag}(a_i, b_i)$, so it
suffices to examine one output. Let $c_i = a_i/b_i$. Then
$S_i A_i S_i^{-1} = c_i A_i$ and $S_i^{-1} B_i = B_i b_i^{-1}$, and since $S_i$ is
symmetric, $S_i^{-1} A_i^\top S_i = c_i A_i^\top$. With $\bar P_i = S_i P_i S_i$,
\begin{align}
  A_i^\top \bar P_i + \bar P_i A_i
  &= c_i\, S_i (A_i^\top P_i + P_i A_i) S_i \\
  &= \bar P_i B_i \big(c_i^{-1} b_i^2 r_i\big)^{-1} B_i^\top \bar P_i - c_i\, S_i Q_i S_i ,
\end{align}
using \eqref{eq:care_i}. Hence $\bar P_i$ solves the CARE for $(A_i, B_i, \bar Q_i, \bar r_i)$ with
\begin{equation}
  \bar Q_i = c_i S_i Q_i S_i 
  \qquad
  \bar r_i = \tfrac{b_i^3}{a_i}\, r_i > 0 .
\end{equation}
Because $\bar Q_i \succ 0$, the CARE has a positive definite solution, which is
stabilizing; as $\bar P_i \succ 0$, it is that solution. Thus $SPS$ is the CARE solution for
$(A_\eta, B_\eta, \bar Q, \bar R)$ with $\bar Q = \mathrm{blkdiag}(\bar Q_i)$,
$\bar R = \mathrm{diag}(\bar r_i)$, and Prop.~1 of \cite{li_clf-rl_2026} applies with
$(Q,R)$ replaced by $(\bar Q, \bar R)$.
\end{proof}

\end{document}